\documentclass[12pt]{amsart} 

\usepackage{amsmath, amssymb, amsthm}  
\usepackage{hyperref}  
\usepackage{graphicx}  
\usepackage{cite}      
\usepackage{geometry}  
\usepackage[numbers]{natbib}
\usepackage{enumerate} 
\usepackage[all]{xy} 
\theoremstyle{plain}  
\newtheorem{theorem}{Theorem}[section]
\newtheorem{lemma}[theorem]{Lemma}
\newtheorem{proposition}[theorem]{Proposition}

\theoremstyle{definition}  
\newtheorem{definition}[theorem]{Definition}
\newtheorem{assumption}[theorem]{Assumption}  
\newtheorem{remark}{Remark}

\usepackage{xcolor}

\title[Learning with $\phi$- and $\beta$-mixing sequence]{Online Regularized Learning Algorithms in RKHS with $\beta$- and $\phi$-Mixing Sequences
}
\author{Priyanka Roy}
\address{Institute for Mathematical Methods in Medicine and Data-Based Modeling\\
         Johannes Kepler University Linz, Altenberger Strasse 69, A-4020 Linz, Austria}
\email{priyanka.roy@jku.at}

\author{Susanne Saminger-Platz}
\address{Institute for Mathematical Methods in Medicine and Data-Based Modeling\\
         Johannes Kepler University Linz, Altenberger Strasse 69, A-4020 Linz, Austria}
\email{susanne.saminger-platz@jku.at}

\keywords{Markov chain, Mixing coefficients, Copula,  Approximation,  Reproducing kernel Hilbert spaces}
\subjclass[2020]{60J20, 68T05, 68Q32, 62L20, 62H05}

\date{\today}

\def\tsc#1{\csdef{#1}{\textsc{\lowercase{#1}}\xspace}}
\tsc{WGM}
\tsc{QE}
\tsc{EP}
\tsc{PMS}
\tsc{BEC}
\tsc{DE}
\newcommand{\opT}[1]{{T}_{K,#1}}
\newcommand{\opI}{{I}}

\begin{document}

\begin{abstract}
In this paper, we study an online regularized learning algorithm in a reproducing kernel Hilbert spaces (RKHS) based on a class of dependent processes. We choose such a process where the degree of dependence is measured by mixing coefficients. As a representative example, we analyze a strictly stationary Markov chain, where the dependence structure is characterized by the \(\phi\)- and \(\beta\)-mixing coefficients. Under these assumptions, we derive probabilistic upper bounds as well as convergence rates for both the exponential and polynomial decay of the mixing coefficients.
\end{abstract}
\maketitle

\section{Introduction}
Consider an iterative algortihm defined by 
\[
f_{t+1} = f_t - \gamma_t \left( (f_t(x_t) - y_t) K_{x_t} + \lambda f_t \right),  \quad \text{with an initial choice } f_1 \in \mathcal{H}_K, 
\]
where $\mathcal{H}_K$ is the reproducing kernel Hilbert space and $\lambda > 0$ is called the \textit{regularization parameter}. We call the sequence $(\gamma_t)_{t \in \mathbb{N}} $ the \textit{step sizes} or \textit{learning rates} and $ (f_t)_{t \in \mathbb{N}}$ the \textit{learning sequence}.

We further consider that the sequence \( (z_t=(x_t,y_t))_{t \in \mathbb{N}} \) is modeled as a strictly stationary Markov chain on an uncountable state space \((Z, \mathcal{B}(Z))\), characterized by a transition kernel \( P \) and a unique stationary distribution \( \rho \), where the decay of dependence in the chain is characterized by the mixing coefficients specifically \(\phi\)- and \(\beta\)-mixing. Based on the Markov samples \( (z_t)_{t \in \mathbb{N}} \), our primary goal is to investigate the convergence behavior of the sequence \( (f_t)_{t \in \mathbb{N}} \) towards the unique minimizer \( f_{\lambda,\mu} \in \mathcal{H}_K \) which takes into account the mixing coefficients of the Markov chain, specifically the \(\phi\)- and \(\beta\)-mixing coefficients. Note that, 
\[
f_{\lambda, \mu}=\underset{f \in \mathcal{H}_K} {\text{arg min}}~\left\{\int_{X} (f(x)-f_{\rho}(x))^2d\mu+\lambda \|f\|_{K}^2\right\},
\] where for probability measure \(\rho\) on \( Z = X \times Y \) and marginal distribution \(\mu\) on \(X\), the regression function \(f_{\rho}\) is given as 
\[
f_{\rho}(x) = \int_Y y \, d\rho(y \mid x),
\]
where \(\rho(y\mid x)\) denotes the conditional distribution of \(y\) given \(x\). It has been be shown that the regression function \( f_{\rho} \) minimizes the expected mean squared error defined by
\[
\mathcal{E}(f) = \int_{X \times Y} (f(x) - y)^2 \, d\rho(z).
\] For more details see Cucker and Zhou~\cite{MR2354721}. Hence, our main goal is to study the probabilistic distance of the current hypothesis \(f_t\) to \(f_{\lambda, \mu}\) and eventually to the regression function \(f_{\rho}\) based on a strictly stationary Markov chain \((z_t)_{t \in \mathbb{N}}\).
\par
For both cases of \(\phi\)- and \(\beta\)-mixing coefficients, a summary of our findings is that for the case of exponential decay of both the mixing coefficients and  for model parameters \(\theta\) and \(\alpha\),  when 
\( \theta \in \left( \tfrac{1}{2}, 1 \right) \), the convergence rate remains 
\( \mathcal{O}\!\left( t^{-\theta/2} \right) \), aligning with the i.i.d.\ rates established by Smale and Yao~\cite{MR2228737}. 
However, at the boundary case \( \theta = 1 \), the rate deteriorates to 
\( \mathcal{O}\!\left( t^{-\alpha/2} \right) \), in contrast to the i.i.d.\ setting where a faster rate 
\( \mathcal{O}\!\left( t^{-\alpha} \right) \) is achieved, as shown by Smale and Yao~\cite{MR2228737}, where \(\alpha= \frac{\lambda}{(\lambda + C_K^2)} \in (0,1]\). 
Moreover, our results are sharp in the sense that when the mixing rate is fast, the dependence becomes negligible, making the error bounds and learning rates nearly identical to those observed in the independent sample setting.\par We further extend the above analysis, for example when $\theta = 1$ and $\lambda < C_K^2$, the rate of convergence is 
\(
 \mathcal{O}\!\left(\lambda^{-1}\,t^{-\frac{\lambda}{2(\lambda + C_K^2)}}\right),
\)
where the convergence exponent $\alpha(\lambda) = \frac{\lambda}{2(\lambda + C_K^2)}$ increases monotonically with $\lambda$, reaching its maximum of $1/2$ as $\lambda \to C_K^2$. Consequently, when $\lambda$ is chosen close to $C_K^2$, the error decays nearly as fast as $t^{-1/2}$, with only a prefactor of order $1/C_K^2$. In contrast, smaller values of $\lambda$ both reduce $\alpha(\lambda)$, slowing the polynomial decay in $t$, and increase the constant factor $\lambda^{-1}$. Thus,  setting $\lambda$ near $C_K^2$ yields a steeper decay in \(t\) at the cost of a moderate constant, whereas taking $\lambda$ very small produces both a larger leading constant and a slower decay in \(t\) comparitively.\par
As for the polynomial decay of both the mixing coefficients, for example, when \(\phi_i\leq bi^{-k}\), where \(b>0\) and \(k>0\), for parameter values 
\( \theta \in \left( \tfrac{1}{2}, 1 \right) \), the convergence rate remains the same as that of the i.i.d. rate for the value \(k>1\), however for \(k=1\), the rate remains almost the same as that of the i.i.d. rate except for a logarithmic factor, i.e., \(\mathcal{O}\!\left(t^{-\theta/2} (\log t)^{1/2} \right)\). 
\par
The version with the independence assumption of this algorithm in reproducing kernel Hilbert space was studied by Smale and Yao~\cite{MR2228737}. 
In many scenarios e.g., time series, Markov chains, stochastic processes the i.i.d.\ assumption is often violated due to the exhibited temporal dependence, see e.g., \cite{mokkadem1988mixing, tuan1985some, beeram2021survey}. Despite this, learning algorithms have been successfully applied in such dependent settings, thereby motivating the need for a more theoretical foundation to understand their performance under these conditions. Various measures of statistical dependence are relevant in this context, for e.g., mixing coefficients (e.g., \(\alpha\)-, \(\beta\)-, and \(\phi\)-mixing), as well as spectral properties such as the spectral gap and mixing time of Markov chains. 
\par

There is growing interest in studying learning algorithms with dependent samples, particularly how different mixing coefficients affect e.g., the stability, optimality of various learning algorithms. A key early contribution by Meir~\cite{meir2000nonparametric} extended the VC framework to \(\beta\)-mixing time series, deriving finite-sample risk bounds, and introducing the use of structural risk minimization for dependent settings. Modha and Masry~\cite{modha2002minimum} extended minimum-complexity regression to dependent samples, retaining i.i.d.~convergence rates for \(m\)-dependent series but slowing under strong mixing due to reduced effective sample size. Subsequent works have characterized the impact of dependence on learning guarantees; for example, Zou and Li~\cite{zou2007performance} established bounds for ERM for exponentially strongly mixing sequences; Mohri and Rostamizadeh~\cite{mohri2008rademacher, mohri2010stability} derived Rademacher-complexity and stability bounds for \(\beta\)- and \(\phi\)-mixing sequences; and Ralaivola~\cite{ralaivola2010chromatic} developed PAC-Bayes bounds via dependency-graph partitioning. Regularized regression methods have likewise advanced to accommodate a range of mixing conditions; for example, Xu and Chen~\cite{MR2406432} showed that for exponentially \(\alpha\)-mixing samples, Tikhonov regularization’s error depends on an “effective” sample size. Sun and Wu~\cite{MR2581234} established capacity-independent least-squares bounds under both 
\(\alpha\)- and \(\phi\)-mixing, achieving up to a log factor the same rates as in the i.i.d.\ setting. Steinwart \emph{et al.}~\cite{steinwart2009learning}, Steinwart and Christmann~\cite{steinwart2009fast}, and Xu \emph{et al.}~\cite{xu2014generalization} analyzed SVM consistency and convergence under various mixing regimes, quantifying the effect of mixing rates on optimality. More recently, Tong and Ng~\cite{tong2024spectral} recovered near-i.i.d.\ rates for strongly mixing sequences based on spectral learning algorithms.
\par
In contrast to previous analyses that predominantly focused on empirical risk minimization, various regularized schemes, etc., under various mixing conditions, our work specifically investigates the performance of online regularized learning algorithm under strictly stationary Markov chain assumptions, where the dependence in the chain is characterized by either $\phi$- or $\beta$-mixing coefficients. We explicitly analyze the learning capabilities of this algorithm by deriving precise convergence rates. In some of the recent works, there have been several investigations to study problem of online learning either reguarized or unregulaized for non i.i.d. samples, for example, Smale and Zhou~\cite{smale2009online} examined an online regularized learning algorithm in RKHS where each training sample is generated by a Markov chain. Exploiting the chain’s exponentially fast mixing i.e., the marginal distribution drifts toward stationarity at an exponential rate, they decompose the sampling process and obtain almost-sure convergence of the algorithm’s iterates to the true regression function. Furthermore, Duchi \emph{et al.}~\cite{agarwal2012generalization} gave generalization bounds for asymptotically stationary mixing processes to the cases of $\beta$- and $\phi$-mixing in the case of stable online learning algorithms. Their method relies on martingale concentration inequalities. Recently, Zhang and Li~\cite{zhang2023online} analyzed online regularized learning algortihm in RKHS for nonparametric regression without the i.i.d. assumption on the samples. They proved that the estimator converges in mean square under two broad non‐i.i.d. settings, i.e., for weakly dependent, non-stationary streams whose instantaneous covariance operators remain persistently excited and for independent samples drawn from drifting probability measures. 
\par
Hence, this work aims to address an open question raised in Ying and Zhou~\cite{MR2302601}, focusing on online regularized learning algorithms in the context of regression problems where the samples exhibit a Markov chain structure.
\par
To outline our paper, we begin in Section~\ref{prelim} with a brief overview of the $\beta$- and $\phi$-mixing coefficients, which characterize the dependence structure in stationary processes. We then provide an overview of the classical framework of online regularized learning algorithms in reproducing kernel Hilbert spaces (RKHS). Section~\ref{sec2} presents our main result, Theorem~\ref{thm:2}, where we derive an upper bound on the initial error and a high-probability bound on the sample error for an online regularized learning algorithm based on a strictly stationary, exponentially $\phi$-mixing Markov chain. To prove Theorem~\ref{thm:2} in the RKHS setting, we first recall a more general result in Section~\ref{section3}. This section recalls a strictly stationary  Markov chain gradient descent algorithm in Hilbert spaces and adapts it to the case of exponentially $\phi$-mixing chains, resulting in Theorem~\ref{thm:1} based on which we finally prove Theorem \ref{thm:2} in Section \ref{secondproof}. Related results for $\beta$-mixing sequences are also discussed in the following section, i.e., in Section \ref{beta1}.
Finally, in Section \ref{eg}, we illustrate an example of computing the relevant dependence coefficient by considering a copula-based Markov chain.
\section{The measures of dependence (mixing coefficients)}\label{prelim}
In this section, we recall some relevant definitions, properties, and established results concerning some of the mixing coefficients (see for e.g.,  \cite{MR2178042, MR2325294, MR2325295, MR2325296}).
\par
Denote $(\Omega, \mathcal{F}, P)$ as a probability space. For any $\sigma$-field $\mathcal{A} \subseteq \mathcal{F}$, let $L^2_{\text{real}}(\mathcal{A})$ denote the space of (equivalence classes of) square-integrable $\mathcal{A}$-measurable (real-valued) random variables. For any two $\sigma$-fields $\mathcal{A}$ and $\mathcal{B} \subseteq \mathcal{F}$, we will focus on the following two measures of dependence
\begin{align*}  
    \phi(\mathcal{A}, \mathcal{B}) &:= \sup \left| P(B \mid A) - P(B) \right|, \quad \forall  A \in \mathcal{A} \text{ with }P(A) > 0, \; \forall B \in \mathcal{B}, \;  \\
    \beta(\mathcal{A}, \mathcal{B}) &:= \sup \frac{1}{2} \sum_{i=1}^{I} \sum_{j=1}^{J} \left| P(A_i \cap B_j) - P(A_i)P(B_j) \right|, \quad \forall A \in \mathcal{A}, \; \forall B \in \mathcal{B},   
\end{align*}
where the supremum is taken over all pairs of (finite) partitions \(\{A_1, \ldots, A_I\}\) and \(\{B_1, \ldots, B_J\}\) of \(\Omega\) such that \(A_i \in \mathcal{A}\) for each \(i \in I\) and \(B_j \in \mathcal{B}\) for each \(j\in J\).\\
For independent \(\mathcal{A}\) and \(\mathcal{B}\) we obtain \begin{align*}
     \phi(\mathcal{A}, \mathcal{B}) = 0,  \quad \beta(\mathcal{A}, \mathcal{B}) = 0.
\end{align*}
Moreover, these measures of dependence satisfy the following inequalities
\begin{align}
    0\leq \beta(\mathcal{A}, \mathcal{B}) \leq \phi(\mathcal{A}, \mathcal{B})\leq 1\label{r1}
\end{align}
We now recall the definition of some mixing coefficients of a strictly stationary Markov chain.
\begin{definition}
    Let $(Z_t)_{t\in\mathbb{N}}$ be a strictly stationary Markov chain; then
$\beta_t = \beta(\sigma(Z_0), \sigma(Z_t))$ and $\phi_t = \phi(\sigma(Z_0), \sigma(Z_t))$. The sequence $(Z_t)_{t\in\mathbb{N}}$ is called \textbf{$\phi$-mixing} if $\phi_t \to 0$ and \textbf{absolutely regular (or $\beta$-mixing)} if $\beta_t \to 0$ for $t\to\infty$.

Using the conditional probabilities \( P^t(z, B) = P(z_t \in B \mid z_0 = z) \) and the invariant distribution \(\rho \) (i.e., we take the starting distribution \(\mu_0=\rho\) since we consider a stationary Markov chain), the mixing coefficients can equivalently be expressed as follows (also see \cite[Theorem 3.32]{MR2325294}, \cite{MR1312160}, \cite{MR2944418})

\begin{equation}
  \beta_t = \int_{Z} \sup_{B \in \mathcal{B}(Z)} \left| P^t(z, B) - \rho(B) \right|\,d\rho,  
\end{equation}

and
\begin{equation}
   \phi_t = \sup_{B \in \mathcal{B}(Z)} \, \text{ess} \sup_{z \in Z} \left| P^t(z, B) - \rho(B) \right|. 
\end{equation}
\end{definition}

\subsection{Example}
Stochastic processes that satisfy mixing conditions include, for instance, ARMA processes or copula-based Markov chains: Mokkaem~\cite{mokkadem1988mixing} demonstrates that, under mild assumptions specifically, the absolute continuity of innovations, every stationary vector ARMA process is geometrically completely regular. Similarly, Longla and Peligrad~\cite{MR2944418} show that for copula-based Markov chains, certain properties such as a positive lower bound on the copula density ensure exponential \(\phi\)-mixing, which in turn implies geometric ergodicity. Such copula-based Markov chains have found applications in time series econometrics and other applied fields. In Section \ref{eg}, we shall further clarify how the mixing coefficients can be estimated for a copula based Markov chain and further desrcibe the problem of learning for such samples. For additional examples of mixing processes, see Davydov~\cite{davydov1973mixing}.
\section{Overview of an online learning algorithm in RKHS}
For a compact metric space \( X\) and a target space \( Y \subseteq \mathbb{R} \), the classical problem of learning from independent and identically distributed (i.i.d.) samples \(( z_i = (x_i, y_i) )_{i=1}^{t} \) drawn according to a probability measure \(\rho\) on \( Z = X \times Y \) aims to approximate the regression function
\[
f_{\rho}(x) = \int_Y y \, d\rho(y \mid x),
\]
where \(\rho(y \mid x)\) denotes the conditional distribution of \(y\) given \(x\). It can be shown (see Cucker and Zhou~\cite{MR2354721}) that the regression function \( f_{\rho} \) minimizes the expected mean squared error defined by
\[
\mathcal{E}(f) = \int_{X \times Y} (f(x) - y)^2 \, d\rho(z).
\]
Let \( K : X \times X \to \mathbb{R} \) be a Mercer kernel, that is, a continuous, symmetric and positive semi-definite function. The RKHS \( \mathcal{H}_K\) associated with \( K\) is defined to be the completion of the linear span of the set of functions \(\{K(\cdot, x) : x \in X\}\) with inner product satisfying, for any \( x \in X \) and \( g \in \mathcal{H}_K \), the reproducing property

\begin{equation}
\langle K_x, g \rangle = g(x).
\label{eq:reproducing_property}
\end{equation}

Let \( (z_t = (x_t, y_t))_{t \in \mathbb{N}} \) be a sequence of random samples independently distributed according to \( \rho \). The \textit{online gradient descent algorithm} is defined as
\begin{equation}
\begin{aligned}
    f_{t+1} &= f_t - \gamma_t\left((f_t(x_t) - y_t) K_{x_t} + \lambda f_t\right)~\text{for some}~f_1\in \mathcal{H}_K~\text{e.g.,}~f_1=0,
\end{aligned}
\label{eq:ogd}
\end{equation}
where \( \lambda \geq 0 \) is called the \textit{regularization parameter}. We call the sequence \( (\gamma_t)_{t \in \mathbb{N}} \) the \textit{step sizes or learning rate} and \( (f_t)_{t \in \mathbb{N}}\) the \textit{learning sequence}. Clearly, each output \( f_{t+1} \) depends on \((z_i)_{i=1}^t\).
\par
Note that by considering a potential loss function given as \( V: W \rightarrow \mathbb{R} \) defined as
\[
V(w) = \frac{1}{2} \langle Aw, w \rangle + \langle B, w \rangle + C,
\]
where \( A: W \rightarrow W \) is a positive definite bounded linear operator with bounded inverse, i.e., \( \|A^{-1}\| < \infty \), \( B \in W \), \( C \in \mathbb{R} \) and \(W\) a Hilbert space, we can derive Eq.~\eqref{eq:ogd} by using the general update formula for gradient descent in a Hilbert space \( W \) given as 
\begin{equation}\label{genupdt}
   w_{t+1} = w_t - \gamma_t \nabla V_{z_t}(w_t). 
\end{equation}
For this, we consider a specific choice of a Hilbert space i.e., $W=\mathcal{H}_K$ and for a fixed $z=(x,y)\in Z$, the quadratic potential map $V:\mathcal{H}_K \rightarrow \mathbb{R}$ defined as 
\begin{equation}\label{lossinhk}
V_{z}(f) = \frac{1}{2} \left( (f(x) - y)^2 + \lambda \|f\|_K^2 \right).
\end{equation}
Hence due to \cite[Proposition~3.1]{MR2228737}, we observe that \(\nabla V_z(f)=(f(x)-y)K_x+\lambda f\) and taking $f=f_t$ and $(x,y)=(x_t,y_t)$, we now obtain that \(\nabla V_{z_t}(f_t)=(f_t(x_t)-y_t)K_{x_t}+\lambda f_t)\). Finally, for a particular choice of a Hilbert space i.e., $W=\mathcal{H}_K$, we identify $w_t=f_t$, $\nabla V_{z_t}(w_t)=\nabla V_{z_t}(f_t)$ in Eq.~\eqref{genupdt} to finally obtain Eq.~\eqref{eq:ogd} i.e., the online learning algortihm in RKHS from the general update formula Eq.~\eqref{genupdt}.
\par
We additionally note that in general the gradient of \(V\) i.e., \(\nabla V: W \rightarrow W\) is given by 
$$\nabla V(w)=Aw+B.$$
Note that for each single sample $z=(x,y)$,
$$\nabla V_z(w)=A(z)w+B(z),$$
where $A(z)$ is a random variable depending on $z$, given by the map $A:Z\rightarrow SL(W)$ taking values in $SL(W)$, the vector space of symmetric bounded linear operators on $W$ and $B:Z\rightarrow W$, where $B(z)$ is a $W$ valued random variable depending on $z$. We shall later clarify the exact expression of \(A:W\rightarrow W\) and \(B\in W\) for the choice of the loss function given as in Eq.~\eqref{lossinhk} in RKHS i.e., when we choose \(W=\mathcal{H}_K\).
\par
Note that in the above context as well as in the rest of the paper, we adopt a slight abuse of notation; i.e., we use $\nabla V_z$ to emphasize dependence on a sample $z \in Z$, and $\nabla V_t$ to emphasize dependence on the time step $t$, referring to $z_t$ when contextually appropriate.
\par
In case \( \lambda > 0 \), we call Eq.~\eqref{eq:ogd} the \textit{online regularized learning algorithm} which has been considered in this paper. In the i.i.d.~setting, this algorithm has been extensively studied in recent literature (see, e.g., \cite{MR2228737, MR2302601}).
\par
We now generalize the online regularized learning algorithm in Eq.~\eqref{eq:ogd} to accommodate dependent observations. Specifically, we consider a strictly stationary Markov chain, where the degree of dependence is quantified by \(\phi\)- and \(\beta\)-mixing coefficients.
\section{An online regularized learning algorithm in RKHS for a strictly stationary exponentially \(\phi\)-mixing Markov chain}\label{sec2}
Let \( X\) be a compact and measurable space, and let \( M > 0 \). Define \( Z = X \times [-M, M] \), which we consider as a subset of \( \mathbb{R}^n \times \mathbb{R} \) equipped with the Borel \(\sigma\)-algebra \(\mathcal{B}(Z)\). Suppose \( (z_t)_{t\in\mathbb{N}} \) is a strictly stationary Markov chain taking values in the state space \( (Z, \mathcal{B}(Z)) \) and is exponentially \(\phi\)-mixing. Let \(\mu\) denote the stationary marginal probability measure on \( X \), and let \(\rho\) denote the stationary joint probability measure on \( Z=X \times Y \).
\begin{definition}
     For a  probability measure $\mu$ we define an integral operator $\opT{\mu}:L^2(\mu)\rightarrow L^2(\mu)$ as 
\begin{equation}\label{eq:Tkmu}
\opT{\mu}f(\cdot)=\int_X K(\cdot,x)f(x)d\mu(x),~~f\in  L^2(\mu),
\end{equation}
where $\opT{\mu}$ is a well-defined continuous and compact operator with $L^2(\mu)$ being the Hilbert space of square-integrable functions with respect to $\mu$. The \textit{regression function} $f_{\rho}$ is said to satisfy the \textbf{source condition} (of order $\nu$) if
$$f_\rho=\opT{\mu}^\nu(g)~\text{for some}~g\in L^{2}(\mu).$$ 
\end{definition}
Under suitable conditions on $K$, the operator $T_{K,\mu}$ is well-defined, continuous, and compact. The compactness of $T_{K,\mu}$ ensures that it has a discrete spectrum with eigenvalues accumulating only at zero.
The source condition $f_{\rho} = T_\mu^\nu g$ essentially quantifies the smoothness of the regression function $f_{\rho}$. Here, $\nu$ indicates the degree of smoothness i.e., higher values of $\nu$ correspond to smoother functions. This condition implies that $f_{\rho}$ lies in the range of $T_\mu^\nu$, which is a subset of $L^2(\mu)$ characterized by functions that are smoother in a certain spectral sense determined by $T_{K,\mu}$. We shall denote \(\|\cdot\|_{\rho}=\|\cdot\|_{L^2({\mu})}\). 
\par
As outlined in the introduction and the beginning of this section, we consider an online regularized algorithm described by Eq.~\eqref{eq:ogd} that operates along a trajectory of a strictly stationary Markov chain whose degree of dependence is characterized by either \(\phi\)- or \(\beta\)-mixing coefficients. The main purpose of this paper is to obtain probabilistic upper bounds for the quantity $\|f_{t+1}-f_{\rho}\|_{\rho}$, based on observations from such a trajectory taking into account the mixing coefficients of the chain.
\par
Note that $f_{t+1}-f_{\rho}$ can be decomposed into the following parts
\begin{align}\label{errdec}
    f_{t+1}-f_{\rho}=\left(f_{t+1}-f_{\lambda, \mu}\right)+\left(f_{\lambda, \mu}-f_{\rho}\right),
\end{align}
where 
\begin{equation}\label{eq:flamstar}
f_{\lambda, \mu}=\underset{f \in \mathcal{H}_K} {\text{arg min}}~\left\{\int_X (f(x)-f_{\rho}(x))^2d\mu+\lambda \|f\|_{K}^2\right\},
\end{equation}
with $\lambda>0$ and $\|f\|_{K}=\|(\opT{\mu}^{1\slash 2})^{-1}f\|_{\rho}$ (see also Smale and Zhou~\cite{MR1959283}). 
Recalling \cite[Proposition~7]{MR1864085}, we obtain 
$$f_{\lambda, \mu}=(\opT{\mu}+\lambda \opI)^{-1}\opT{\mu}f_{\rho},$$
as a unique minimizer of Eq.~\eqref{eq:flamstar}.
 \subsection{Upper bounding the distance between $f_{\lambda, \mu}~\text{and}~f_{\rho}$ }
 In order to estimate the term $\|f_{\lambda, \mu}-f_{\rho}\|_{\rho}$, we recall Lemma 3 in \cite{MR2327597} i.e.,
\begin{lemma}[\cite{MR2327597}]
    Let $\lambda>0$ and $f_{\lambda, \mu}$ be defined by Eq.~\eqref{eq:flamstar}. If $f_\rho$ satisfies the source condition (of order $\nu$) with $0<\nu\leq 1$, then for some $g\in L^2(\mu)$ and for any $\lambda >0$, the following holds
    $$\|f_{\lambda, \mu}-f_\rho\|_{\rho}\leq \lambda^\nu \|g\|_{\rho}. $$
\end{lemma}
\subsection{The upper bound on $\|f_{t+1}-f_{\lambda, \mu}\|_{\rho}$}\label{firstupperbound}
We now consider a quadratic potential map \( V_{z}: \mathcal{H}_K \to \mathbb{R} \), defined by
\[
V_z(f) = \frac{1}{2} \left( (f(x) - y)^2 + \lambda \|f\|_K^2 \right).
\]
In Section \ref{secondproof}, it is shown that the gradient \( \nabla V_z(f) \) satisfies, for all \( z \in Z \), the following inequalities
\[
\|\nabla V_z(f_{\lambda, \mu})\|^2 \leq \left( \frac{2M C_K^2(\lambda + C_K^2)}{\lambda} \right)^2
\]
and
\[
\lambda \leq \nabla^2 V_z(f) \leq \lambda + C_K^2,
\]
thus satisfies Assumptions \ref{A1} and \ref{A2} in Section \ref{section3}.
We denote \[C_K = \underset{x \in X}{\sup} \sqrt{K(x, x)}
\] and assume that \( C_K < \infty \). Note that \(
\|f_{t+1} - f_{\lambda, \mu}\|_{\rho} = C_K \|f_{t+1} - f_{\lambda, \mu}\|_K.\)\par
We consider an exponentially \(\phi\)-mixing sequence \( (z_t)_{t \in \mathbb{N}} \), meaning the sequence is geometrically ergodic i.e., there exist constants \( D > 0 \) and \( 0 < r < 1 \) such that
\[
\beta_t \leq \phi_t \leq D r^t,
\]
see also Bradley~\cite{MR2178042}, where \( z_t \in Z \subseteq X \times [-M, M] \) for all \( t \in \mathbb{N} \) with stationary distribution \( \rho \).
Given the learning rate \(
\gamma_t = \frac{1}{(\lambda + C_K^2)t^\theta},
\) with \( C_K\) defined as above, we involve Theorem \ref{thm:1} and Proposition \ref{rm:1} (see Section \ref{section3} for results) to derive upper bounds on the error term \( \|f_t - f_{\lambda, \mu}\|_K \) for the cases \( \theta \in \left( \frac{1}{2}, 1 \right) \) and \( \theta = 1 \).

\begin{theorem} \label{thm:2}
Let \( Z = X \times [-M, M] \) for some \( M > 0 \), and let \((z_t)_{t\in\mathbb{N}}\) be a strictly stationary Markov chain on \((Z, \mathcal{B}(Z))\) such that it is \(\phi\)-mixing at least exponentially fast i.e., there exist constants \( D > 0 \) and \( 0 < r < 1 \) such that \( \beta_t \leq \phi_t \leq D r^t \). Let $\theta \in (\frac{1}{2},1)~\text{and}~\lambda>0$ and consider $\gamma_t= \frac{1}{(\lambda+C_K^2)t^\theta}$ and $\alpha=\frac{\lambda}{(\lambda +C_K^2)}\in (0,1]$. Then we have, for each \(t\in \mathbb{N}\) and \(f_t\) obtained by Eq.~\eqref{eq:ogd},
$$ \|f_{t}-f_{\lambda, \mu}\|_{K}\leq \mathcal{E}_\text{init}(t)+\mathcal{E}_\text{samp}(t),
$$
where 
$$\mathcal{E}_\text{init}(t)\leq e^{\frac{2\alpha}{1-\theta}(1-t^{1-\theta})}\|f_{1}-f_{\lambda, \mu}\|_K;$$
and with probability at least $1-\delta$, with $\delta \in(0,1)$ in the space $Z^{t-1}$,
$$\mathcal{E}^2_\text{samp}(t) \leq  \frac{c'C_{\theta}}{ \delta\lambda^2 }\left( \frac{1}{\alpha} \right)^{\theta / (1-\theta)}\left( \frac{1}{t} \right)^{\theta}\left(1+\frac{4Dr}{1-r}\right),$$
with $ C_{\theta}= \left(8 + \frac{2}{2\theta - 1} \left( \frac{\theta}{e(2 - 2^\theta)} \right)^{\theta / (1-\theta)}\right) $ and $c'=4(MC_K^2)^2$.
    \end{theorem}
\begin{proposition} \label{rm}
    With all the assumptions as in Theorem \ref{thm:2},  but with $\theta=1$ and $\lambda<C_K^2$ for which $\alpha=\frac{\lambda}{\lambda+C_K^2}\in \left(0,\frac{1}{2}\right)$, we obtain that 
$$\|f_{t}-f_{\lambda,\mu}\|_{K}\leq \mathcal{E}_\text{init}(t)+\mathcal{E}_\text{samp}(t)$$
where,
$$\mathcal{E}_\text{init}(t)\leq \left(\frac{1}{t}\right)^\alpha\|f_{1}-f_{\lambda,\mu}\|_{K}; $$
and with probability at least $1-\delta,~\text{with}~\delta \in(0,1)$ in the space $Z^{t-1}$,
$$\mathcal{E}^2_\text{samp}(t)\leq \frac{4c'}{\delta\lambda^2}\left(\frac{1}{1-2\alpha}\right)\left(\frac{1}{t}\right)^\alpha\left(1+\frac{6Dr}{1-r}\right),$$
with $c'=4(MC_K^2)^2$. 
\end{proposition}
Note that in Proposition \ref{rm}, we can obtain similar results for different values of \(\alpha \in (0,1]\) i.e., when \(\alpha = \frac{1}{2}\), the condition simplifies to \(\lambda = C_K^2\). For \(\alpha \in \left(\frac{1}{2},1\right)\), the condition becomes \(\lambda > C_K^2\), and for \(\alpha = 1\), we have \(C_K^2 = 0\).
\\[1ex]
We defer the proofs to Section \ref{secondproof}.

\begin{remark} 
 Assuming without loss of generality that \( 0 < \lambda \leq 1 \), Theorem \ref{thm:2} implies the following rate of convergence
\[
\| f_t - f_{\lambda,\mu} \|_K = \mathcal{O}\!\left(\, \lambda^{-\tau(\theta)}\, t^{-\theta/2}\right),
\qquad
\tau(\theta) = \frac{2 - \theta}{2(1 - \theta)} = \frac{1}{2(1 - \theta)} + \frac{1}{2},
\]
for every \( \theta \in \left( \tfrac{1}{2}, 1 \right) \). The function \( \tau(\theta) \) is strictly increasing on this interval, so \( \tau(\theta) \in \left( \tfrac{3}{4}, \infty \right) \). 

Hence, when \( \lambda \) is small, selecting \( \theta \) close to \( \tfrac{1}{2} \) yields the following rate of convergence
\[
\| f_t - f_{\lambda,\mu} \|_K = \mathcal{O}\!\left( \lambda^{-3/4} t^{-1/4} \right),
\]
which is tight in \( \lambda \), but loose in \( t \). Whereas when \( \theta \) increases, the convergence rate becomes tighter in \(t\), but looser in \(\lambda\). This analysis is similar to the discussion in Smale and Yao~\cite{MR2228737} since for the choice of \(\theta \in (\frac{1}{2},1)\), the rate remains similar to that of the i.i.d.~case.
\end{remark}
\begin{remark}
    For the case $\theta = 1$ and $\lambda < C_K^2$, Proposition \ref{rm} gives the following rate of convergence
\[
\|f_t - f_{\lambda,\mu}\|_K = \mathcal{O}\!\left(\lambda^{-1}\,t^{-\frac{\lambda}{2(\lambda + C_K^2)}}\right).
\]
Observe that the convergence exponent $\alpha(\lambda) = \frac{\lambda}{2(\lambda + C_K^2)}$ increases monotonically with $\lambda$, reaching its maximum of $1/2$ as $\lambda \to C_K^2$. Consequently, when $\lambda$ is chosen close to $C_K^2$, the error decays nearly as fast as $t^{-1/2}$, with only a prefactor of order $1/C_K^2$. In contrast, smaller values of $\lambda$ both reduce $\alpha(\lambda)$, slowing the polynomial decay in $t$, and increase the constant factor $\lambda^{-1}$. Thus on setting $\lambda$ near $C_K^2$ yields a steeper decay in \(t\) at the cost of a moderate constant, whereas taking $\lambda$ very small produces both a larger leading constant and a slower decay in \(t\) comparitively. Further, note that at the boundary value \( \theta = 1 \), the rate slows to \( \mathcal{O}\!\left( t^{-\alpha/2} \right) \), whereas the i.i.d.\ scenario attains the faster rate \( \mathcal{O}\!\left( t^{-\alpha} \right) \) as shown by Smale and Yao~\cite{MR2228737} where, \( \alpha=\frac{\lambda}{\lambda+C_K^2} \in (0,1] \).
\end{remark}

Before presenting the proof of Theorem \ref{thm:2}, we first recall the general framework of the strictly stationary Markov chain gradient descent algorithm; for further details, see Roy and Saminger-Platz~\cite{roy2025gradientdescentalgorithmhilbert}. We then recall some results that are essential in establishing the proof of Theorem \ref{thm:2}.
\par
 In this framework, the algorithm operates along the trajectory of a strictly stationary Markov chain \((z_t)_{t \in \mathbb{N}}\) within an uncountable state space \(Z,\mathcal{B}(Z)\) that admits a unique stationary distribution \(\rho\). The algorithm iteratively updates the function
\[
w_{t+1} = \mathcal{A'}(w_t, z_t),
\]  
where, \(\mathcal{A'}\) represents a strictly stationary Markov chain gradient descent algorithm where we modify the algorithm by considering the starting distribution as the chain's stationary distribution. Formally,  
\[
\mathcal{A'}: W \times Z \to W,
\]  
with \(W\) denoting a Hilbert space. Note that this approach defines a variant of the stochastic gradient descent algorithm, which we refer to as strictly stationary Markov chain gradient descent (SS-MGD). Unlike conventional Markov chain gradient descent methods that emphasize mixing time and convergence toward the stationary distribution see, e.g., \cite{sun2018Markov,nagaraj2020least,dorfman2022adapting,even2023stochastic}, SS-MGD assumes that the Markov chain is strictly stationary from initialization, i.e., the initial distribution is already the invariant distribution \(\rho\).

In this setting, we recall the upper bounds on the distance between the approximation \(w_t\) and the optimal solution \(w^\star\) as shown by Roy and Saminger-Platz~\cite{roy2025gradientdescentalgorithmhilbert}. Additionally, we extend these results to bound the distance between the approximation \(f_t\) and the optimal solution \(f_{\lambda, \mu}\) within a specific choice of Hilbert space i.e., \(W = \mathcal{H}_K\), as discussed previously in Subsection \ref{firstupperbound}. Here, the updates follow the form
\[
f_{t+1} = \mathcal{A'}(f_t, z_t),
\]  
where \(\mathcal{A'}\) functions as an online regularized learning algorithm based on the strictly stationary Markov chain trajectory where,
\[
\mathcal{A'}: \mathcal{H}_K \times Z \to \mathcal{H}_K.
\]
\section{Strictly stationary Markov chain gradient descent algorithm in Hilbert spaces}\label{section3}
Let \( W \) be a Hilbert space, and let \( (z_t)_{t \in \mathbb{N}} \) be a strictly stationary Markov chain on the measurable space \( (Z, \mathcal{B}(Z)) \), with transition kernel \( P \) and unique stationary distribution \( \rho \). Furthermore, we consider that the decay of dependence exhibited in the Markov chain is characterized by certain mixing coefficients, specifically \(\phi\)- and \(\beta\)-mixing. Given a quadratic loss function \( V : W \times Z \to \mathbb{R} \), we denote its gradient with respect to the first argument by \( \nabla V_z(w) \). Starting from an initial point \( w_1 \in W \), we define an iterative sequence by

\begin{equation}
    w_{t+1} = w_t - \gamma_t \nabla V_{z_t}(w_t), \quad t \in \mathbb{N},
    \label{eq:stograd}
\end{equation}
where \( (\gamma_t)_{t \in \mathbb{N}} \) is a positive step-size sequence. Note that,  \( V: W \rightarrow \mathbb{R} \) is given by
\[
V(w) = \frac{1}{2} \langle Aw, w \rangle + \langle B, w \rangle + C,
\]
where \( A: W \rightarrow W \) is a positive definite bounded linear operator with bounded inverse, i.e., \( \|A^{-1}\| < \infty \), \( B \in W \), and \( C \in \mathbb{R} \). Let \( w^{\star} \in W \) denote the unique minimizer of \( V \) such that \( \nabla V(w^{\star}) = 0 \).
\par
Based on the Markov samples \( (z_t)_{t \in \mathbb{N}} \), we now recall some relevant results that investigate the convergence behavior of the sequence \( (w_t)_{t \in \mathbb{N}} \) toward the unique minimizer \( w^\star \in W \) of the quadratic loss function \( V \), satisfying
\(
    \nabla V(w^\star) = 0,
\) which takes into account the mixing coefficients of the Markov chain, specifically the \(\phi\)- and \(\beta\)-mixing coefficients under the following assumptions
\begin{assumption}\label{A1}
The function \( V \) has a unique minimizer \( w^\star \), and for all \( z \in Z \), there exists a constant \( \sigma \geq 0 \) such that
\[
\|\nabla V_z(w^\star)\|^2 \leq \sigma^2.
\]
\end{assumption}

\begin{assumption}\label{A2}
For all \( z \in Z \), the function \( V_z \) are \( \eta \)-smooth and \( \kappa \)-strongly convex, i.e., for all \( w \in W \)
\[
\kappa  \leq \nabla^2 V_z(w) \leq \eta .
\]
Define $\alpha=\frac{\kappa}{\eta}$ where $\alpha \in (0,1].$
\end{assumption}
Note that Assumption \ref{A1} reflects the noise at optimum with mean 0 i.e., $\mathbb{E}[\nabla V_z(w^\star)]=0.$ 
\begin{theorem}[\cite{roy2025gradientdescentalgorithmhilbert}] \label{thm:1}
Let us consider a strictly stationary Markov chain \((z_t)_{t\in\mathbb{N}}\) such that it is \(\phi\)-mixing at least exponentially fast i.e., there exist constants \( D > 0 \) and \( 0 < r < 1 \) such that \( \beta_t \leq \phi_t \leq D r^t \). Furthermore,
let $\theta\in\left(\tfrac12,1\right)$ and consider $\gamma_t=\frac{1}{\eta t^\theta}$. Then under Assumptions \ref{A1} and \ref{A2}, for each $t\in \mathbb{N}$ and \(w_t\) obtained by Eq.~\eqref{eq:stograd}, we have
$$\|w_{t}-w^{\star}\| \leq \mathcal{E}_\text{init}(t)+\mathcal{E}_\text{samp}(t)$$
where,
$$\mathcal{E}_\text{init}(t)\leq e^{\frac{2\alpha}{1-\theta}(1-t^{1-\theta})}\|w_{1}-w^{\star}\|;$$
and with probability at least $1-\delta$, with $\delta \in(0,1)$ in the space $Z^{t-1}$,
$$\mathcal{E}^2_\text{samp}(t) \leq  \frac{\sigma^2C_{\theta}}{\delta \eta^2}\left(\frac{1}{\alpha} \right)^{\theta / (1-\theta)}\left( \frac{1}{t} \right)^{\theta}\left(1+\frac{4Dr}{1-r}\right),$$

with $ C_{\theta}= \left(8 + \frac{2}{2\theta - 1} \left( \frac{\theta}{e(2 - 2^\theta)} \right)^{\theta / (1-\theta)}\right).$
    \end{theorem}
    \begin{proposition}[\cite{roy2025gradientdescentalgorithmhilbert}] \label{rm:1}
    With all the assumptions of Theorem \ref{thm:1}, but with $\theta=1$ and $\alpha\in \left(0,\frac{1}{2}\right)$, we obtain that
    $$\|w_{t}-w^{\star}\|\leq  \mathcal{E}_\text{init}(t)+\mathcal{E}_\text{samp}(t)$$
    where,
$$\mathcal{E}_\text{init}(t)\leq \left(\frac{1}{t}\right)^\alpha\|w_{1}-w^{\star}\|;$$
and with probability at least $1-\delta,~\text{with}~\delta \in(0,1)$ in the space $Z^{t-1}$,
$$\mathcal{E}^2_\text{samp}(t)\leq  \frac{4\sigma^2}{\delta\eta^2}\left(\frac{1}{1-2\alpha}\right)\left(\frac{1}{t}\right)^\alpha\left(1+\frac{6Dr}{1-r}\right).$$
\end{proposition}
\begin{remark}\label{remark1}
As discussed by Roy and Saminger-Platz~\cite{roy2025gradientdescentalgorithmhilbert}, in the finite sample setting, the exponential decay rate determined by the mixing coefficient influences the error bounds we obtain, in contrast to the i.i.d.\ case treated by Smale and Yao~\cite{MR2228737}. When \( r \to 0 \), the process mixes more rapidly and the factor \( \tfrac{r}{1 - r} \) tends to zero, so our error bounds essentially coincide with those for independent samples and hence our bounds are sharper in this sense. Stronger dependence reduces the effective information and leads to worst error bounds compared to i.i.d.
Turning to the convergence rates, when the parameter \( \theta \in \left( \tfrac{1}{2}, 1 \right) \), we still achieve the rate \( \mathcal{O}\!\left( t^{-\theta/2} \right) \), matching the i.i.d.\ rate of Smale and Yao~\cite{MR2228737}. However at the boundary value \( \theta = 1 \), the rate slows to \( \mathcal{O}\!\left( t^{-\alpha/2} \right) \), whereas the i.i.d.\ scenario attains the faster rate \( \mathcal{O}\!\left( t^{-\alpha} \right) \) as shown by Smale and Yao~\cite{MR2228737} where, \( \alpha = \in (0,1] \).
\end{remark}
    With the upper bounds on the distance between the estimate \(w_t\) and the optimal solution \(w^\star\) now stated, we are ready to prove Theorem \ref{thm:2}. This theorem provides an upper bound on the distance between the estimate \(f_t\) and the optimal solution \(f_{\lambda, \mu}\) in the Hilbert space \(\mathcal{H}_K\) for \(\theta \in \left(\frac{1}{2}, 1\right)\). For this, we identify \(w_t\) with \(f_t\), \(w^\star\) with \(f_{\lambda, \mu}\), and \(W\) with \(\mathcal{H}_K\).
\section{Proof of Theorem \ref{thm:2}}\label{secondproof}
\begin{proof}
    Let \( W = \mathcal{H}_K \). Define \( J_x \) as the evaluation functional such that, for any \( x \in X \subset \mathbb{R}^d\), \( J_x : \mathcal{H}_K \rightarrow \mathbb{R} \), where \( J_x(f) = f(x) \) for all \( f \in \mathcal{H}_K \). Let \( J_x^\star \) denote the adjoint operator of \( J_x \), i.e., \( J_x^\star : \mathbb{R} \rightarrow \mathcal{H}_K \). Consequently, we have:
\[
\langle J_x(f), y \rangle _\mathbb{R} = \langle f(x), y \rangle _\mathbb{R} = y f(x) = y \langle f, K_x \rangle _{\mathcal{H}_K} = \langle f, y K_x \rangle _{\mathcal{H}_K}.
\]
Moreover, since \( \langle J_x(f), y \rangle _\mathbb{R} = \langle f, J_x^\star(y) \rangle _{\mathcal{H}_K} \), we obtain that \( J_x^\star(y) = yK_x \).

Next, define the linear operator \( A_x : \mathcal{H}_K \rightarrow \mathcal{H}_K \) as \( A_x = J_x^\star J_x + \lambda I \), where \( I\) is the identity operator. Then, for any \( f \in \mathcal{H}_K \), we have
\[
A_x(f) = J_x^\star J_x(f) + \lambda f = f(x)K_x + \lambda f.
\]
Thus, \( A_x \) is a random variable depending on \( x \). Taking the expectation of \( A_x \) over \( x \), we get:
\[
\hat{A} = \mathbb{E}_x \left[ A_x \right] = \opT{\mu} + \lambda I,
\]
where \( \opT{\mu} \) is the integral operator associated with the measure \( \mu \) on \( \mathcal{H}_K \). Also, since \( J_x^\star J_x : \mathcal{H}_K \rightarrow \mathcal{H}_K \) is given by \( J_x^\star J_x(f) = f(x)K_x \), it follows that:
\[
\opT{\mu} = \mathbb{E}_x \left[ J_x^\star J_x \right].
\]

Next, define \( B_z = J_x^\star (-y) = -yK_x \in \mathcal{H}_K \), where \( B_z \) is a random variable depending on \( z = (x, y) \). Taking the expectation of \( B_z \), we obtain:
\[
\hat{B} = \mathbb{E}_z \left[ B_z \right] = \mathbb{E}_z \left[ -yK_x \right] = \mathbb{E}_x \left[ \mathbb{E}_y \left[ -y \right] K_x \right] = \mathbb{E}_x \left[ -f_\rho K_x \right] = -\opT{\mu} f_\rho.
\]

Recall that \( V : \mathcal{H}_K \rightarrow \mathbb{R} \) is a quadratic potential map with the general form
\[
V(f) = \frac{1}{2} \langle Af, f \rangle + \langle B, f \rangle + C,
\]
where \( A : \mathcal{H}_K \rightarrow \mathcal{H}_K \) is a positive, bounded linear operator with \( \| A^{-1} \| < \infty \), \( B \in \mathcal{H}_K \), and \( C \in \mathbb{R} \). Consequently, the gradient of \( V \), \( \nabla V : \mathcal{H}_K \rightarrow \mathcal{H}_K \), is given by
\[
\nabla V(f) = Af + B.
\]

For each sample \( z \), we define
\[
\nabla V_z(f) = A(z) f + B(z) = A_x f + B_z,
\]
where \( A(z) \) is a random variable depending on \( z = (x, y) \), given by the map \( A : Z \rightarrow SL(\mathcal{H}_K) \) that takes values in \( SL(\mathcal{H}_K) \), the vector space of symmetric bounded linear operators on \( \mathcal{H}_K \), and \( B : Z \rightarrow \mathcal{H}_K \) such that \( B(z) \) is a \( \mathcal{H}_K \)-valued random variable depending on \( z \). Moving from the general setting to a specific case where
\[
V_z(f) = \frac{1}{2} \left( (f(x) - y)^2 + \lambda \| f \|_k^2 \right),
\]
and using the fact that \( f_{\lambda, \mu} = (\opT{\mu} + \lambda I)^{-1} \opT{\mu} f_\rho \), we obtain:
\[
\mathbb{E}_z \left[ \nabla V_z(f_{\lambda, \mu}) \right] = \mathbb{E}_z \left[ A_z f_{\lambda, \mu} + B_z \right] = \mathbb{E}_x \left[ A_x \right] f_{\lambda, \mu} + \mathbb{E}_z \left[ B_z \right].
\]
Substituting the expectations, we get
\[
\mathbb{E}_z \left[ \nabla V_z(f_{\lambda, \mu}) \right] = (\opT{\mu} + \lambda I) (\opT{\mu} + \lambda I)^{-1} \opT{\mu} f_\rho - \opT{\mu} f_\rho = 0.
\]
Hence, in expectation, \( f_{\lambda, \mu} \) is a minimizer of \( V_z \) for \( z \in Z \).

Assumption \ref{A2} translates into
\[
\nabla^2 V_z(f) \leq \eta \quad \text{and} \quad \nabla^2 V_z(f) \geq \alpha \implies A_x \leq \eta \quad \text{and} \quad A_x \geq \kappa.
\]
Since \( \| \opT{\mu} \| \leq C_K^2 \) (see, e.g., Cucker and Zhou~\cite{MR2354721}), where \( C_K = \underset{x \in X}{\sup} \sqrt{K(x, x)} \), and \( \mathbb{E}_x \left[ A_x \right] = \opT{\mu} + \lambda I \), we obtain that \( \kappa= \lambda \) and \( \eta = \lambda + C_K^2 \).

Moreover, we have
\[
\| A_z f_{\lambda, \mu} + B_z \| \leq \| A_x \| \| f_{\lambda, \mu} \| + \| B_z \| \leq (\lambda + C_K^2) \| (\opT{\mu} + \lambda I)^{-1} \opT{\mu} f_\rho \| + \| -yK_x \|.
\]
Simplifying further
\[
\| A_z f_{\lambda, \mu} + B_z \| \leq (\lambda + C_K^2) \| \hat{A}^{-1} \hat{B} \| + M C_K^2 \leq (\lambda + C_K^2) \frac{1}{\lambda} M C_K^2 + M C_K^2 \leq \frac{2M C_K^2(\lambda + C_K^2)}{\lambda}.
\]
Thus, we conclude that
\[
\| A_z f_{\lambda, \mu} + B_z \|^2 \leq \left( \frac{2M C_K^2(\lambda + C_K^2)}{\lambda} \right)^2,
\]
which verifies Assumption \ref{A1}.
\par
We now finally identify $f_t=w_t$, $f_{\lambda,\mu}=w^\star$, $\gamma_t= \frac{1}{(\lambda+C_K^2)}\frac{1}{t^\theta}$ with $\theta \in (\frac{1}{2},1)$, $\sigma^2=\left(\frac{2M C_K^2(\lambda + C_K^2)}{\lambda}\right)^2$ and $W=\mathcal{H}_K$ in Theorem \ref{thm:1} and obtain an upper bound for the initial error, $\mathcal{E}_\text{init}(t)$,  as $\mathcal{E}_\text{init}(t)\leq e^{\frac{2\alpha}{1-\theta}(1-t^{1-\theta})}\|f_{1}-f_{\lambda, \mu}\|_\rho$ and a probabilistic upper bound for the sampling error, $\mathcal{E}_\text{samp}(t)$, as $$\mathcal{E}^2_\text{samp}(t) \leq  \frac{c'C_{\theta}}{ \lambda^2 \delta}\left( \frac{1}{\alpha} \right)^{\theta / (1-\theta)}\left( \frac{1}{t} \right)^{\theta}\left(1+\frac{4Dr}{1-r}\right),$$
with $ C_{\theta}= \left(8 + \frac{2}{2\theta - 1} \left( \frac{\theta}{e(2 - 2^\theta)} \right)^{\theta / (1-\theta)}\right)$, where $c'=4(MC_K^2)^2$, $\alpha=\frac{\lambda}{\lambda+C_K^2}$ and $\delta \in (0,1).$
\end{proof}
Using similar reasoning as above, the result of Proposition~\ref{rm} for the case \(\theta=1\) and  $\alpha\in \left(0,\frac{1}{2}\right)$ follows directly from Proposition~\ref{rm:1}, which provides an upper bound on the distance between \(w_t\) and \(w^\star\) when \(\theta=1\). 
\par
\section{Strictly stationary exponentially \(\beta\)-mixing Markov chain}\label{beta1}
We first recall the following theorem.
\begin{theorem}[\cite{roy2025gradientdescentalgorithmhilbert}] \label{beta}
Let us consider a strictly stationary Markov chain \((z_t)_{t\in\mathbb{N}}\) such that it is exponentially \(\beta\)-mixing i.e., there exist constants \( D_1 > 0 \) and \( 0 < r_1 < 1 \) such that \( \beta_t \leq D_1 r_1^t \). Furthermore,
let $\theta\in\left(\tfrac12,1\right)$ and consider $\gamma_t=\frac{1}{\eta t^\theta}$. Then under Assumptions \ref{A1} and \ref{A2}, for each $t\in \mathbb{N}$ and \(w_t\) obtained by Eq.~\eqref{eq:stograd}, we have
$$\|w_{t}-w^{\star}\| \leq \mathcal{E}_\text{init}(t)+\mathcal{E}_\text{samp}(t)$$
where,
$$\mathcal{E}_\text{init}(t)\leq e^{\frac{2\alpha}{1-\theta}(1-t^{1-\theta})}\|w_{1}-w^{\star}\|;$$
and with probability at least $1-\delta$, with $\delta \in(0,1)$ in the space $Z^{t-1}$,
$$\mathcal{E}^2_\text{samp}(t) \leq  \frac{\sigma^2C_{\theta}}{\delta' \eta^2}\left( \frac{1}{\alpha} \right)^{\theta / (1-\theta)}\left( \frac{1}{t} \right)^{\theta}\left(1+\frac{4D_1r_1}{1-r_1}\right),$$
with $ C_{\theta}= \left(8 + \frac{2}{2\theta - 1} \left( \frac{\theta}{e(2 - 2^\theta)} \right)^{\theta / (1-\theta)}\right).$
    \end{theorem}

 Hence with respect to the above Theorem i.e.,  Theorem \ref{beta}, the proof follows similarly as that of Theorem \ref{thm:2} for \(\beta\)-mixing sequences as well. That is, under the same assumptions as in Theorem \ref{thm:2}, but for a strictly stationary Markov chain \((z_t)_{t\in\mathbb{N}}\) on \((Z, \mathcal{B}(Z))\), where the chain is \(\beta\)-mixing at least exponentially fast (i.e., \(\beta_t \leq D_1r_1^t\) for some constant \(D_1 > 0\) and \(0 < r_1 < 1\)), we obtain the following result on the sampling error i.e., with probability at least \(1 - \delta\), where \(\delta \in (0,1)\) in the space \(Z^{t-1}\),
\[
\mathcal{E}^2_\text{samp}(t) \leq \frac{c'C_{\theta}}{\lambda^2 \delta}\left( \frac{1}{\alpha} \right)^{\theta / (1-\theta)} \left( \frac{1}{t} \right)^{\theta} \left(1+\frac{4D_1r_1}{1-r_1}\right),
\]
where \(
C_{\theta} = \left(8 + \frac{2}{2\theta - 1} \left( \frac{\theta}{e(2 - 2^\theta)} \right)^{\theta / (1-\theta)}\right)\), $\alpha=\frac{\lambda}{\lambda+C_K^2}$  and  \(c' = 4(MC_K^2)^2.\)
\section{Learning rates for polynomial decay of \(\phi\)-mixing coefficient}
Based on all the previous dicussions and problem setup, in this section we discuss the learning rates of  \(f_t\) obtained by Eq.~\eqref{eq:ogd},
to \(f_{\lambda, \mu}\) specifically when the \(\phi\)-mixing coefficients have a polynomial decay. For that, we first recall the following theorem
 \begin{theorem}[\cite{roy2025gradientdescentalgorithmhilbert}]
       Let us consider a strictly stationary Markov chain \((z_t)_{t\in\mathbb{N}}\) such that it is \(\phi\)-mixing  satisfying a polynomial decay, i.e., \(\phi_t\leq bt^{-k}\) for \(b>0\) and \(k>0\). Furthermore,
let $\theta\in\left(\tfrac12,1\right)$ and consider $\gamma_t=\frac{1}{\eta t^\theta}$. Then under Assumptions \ref{A1} and \ref{A2}, for each $t\in \mathbb{N}$ and \(w_t\) obtained by Eq.~\eqref{eq:stograd}, we have
$$\|w_{t}-w^{\star}\|
 =
\begin{cases}
\mathcal{O}\!\left(t^{\frac{1-k-\theta}{2}}\right), & 0<k<1,\\[8pt]
\mathcal{O}\!\left(t^{-\theta/2} (\log t)^{1/2} \right), & k=1,\\[8pt]
\mathcal{O}\!\left(t^{-\theta/2}\right), & k>1.
\end{cases}$$
    \end{theorem}
  Specialising to the case \( W = \mathcal{H}_K \), a reproducing kernel Hilbert space, the preceding result remains valid. Indeed, the arguments in the proof of Theorem~\ref{thm:2} show that the \(\eta\) smoothness in Assumption~\ref{A2} reduces to the simple identification of \(
 \eta = \lambda + C_K^2\), where \(\lambda>0\) is the regularization parameter and \( C_K = \underset{x \in X}{\sup} \sqrt{K(x, x)} \). Hence, we are now able to  state the following theorem
     \begin{theorem}
       Let \( Z = X \times [-M, M] \) for some \( M > 0 \), and let \((z_t)_{t\in\mathbb{N}}\) be a strictly stationary Markov chain on \((Z, \mathcal{B}(Z))\) such that it is \(\phi\)-mixing satisfying a polynomial decay, i.e., \(\phi_t\leq bt^{-k}\) for \(b>0\) and \(k>0\). Let $\theta \in (\frac{1}{2},1)~\text{and}~\lambda>0$ and consider $\gamma_t= \frac{1}{(\lambda+C_K^2)t^\theta}$. Then we have, for each \(t\in \mathbb{N}\) and \(f_t\) obtained by Eq.~\eqref{eq:ogd},
$$\|f_{t}-f_{\lambda, \mu}\|
 =
\begin{cases}
\mathcal{O}\!\left(t^{\frac{1-k-\theta}{2}}\right), & 0<k<1,\\[8pt]
\mathcal{O}\!\left(t^{-\theta/2} (\log t)^{1/2} \right), & k=1,\\[8pt]
\mathcal{O}\!\left(t^{-\theta/2}\right), & k>1.
\end{cases}$$
    \end{theorem}
 Note that the convergence rate remains the same as that of the i.i.d.~(see Smale and Yao~\cite{MR2228737}) rate for the value \(k>1\), however for \(k=1\), the rate remains almost the same as that of the i.i.d. rate except for a logarithmic factor, i.e., \(\mathcal{O}\!\left(t^{-\theta/2} (\log t)^{1/2} \right)\).
\par
Similar conclusions follow for a polynomially \(\beta\)-mixing sequence as well. For more details, see Roy and Saminger-Platz~\cite{roy2025gradientdescentalgorithmhilbert}.
\section{Example}\label{eg}
\subsection{Copula-based Markov Chain}
Copula-based Markov chains are valuable tools for estimating mixing coefficients. By decoupling the dependence structure (captured by the copula) from the marginal distributions, this method simplifies the study of mixing behavior. Mixing coefficients can be computed using the copula and its density. Utilizing copula theory, we can investigate various dependence properties directly linked to the mixing behavior of the chain. For instance, certain characteristics of the copula, such as a positive lower bound on its density, imply exponential \(\phi\)-mixing thereby geometric ergodicity (see \cite[Theorem 8]{MR2944418} and also Bradley~\cite{MR2178042} for more details).

In practice, these coefficients are estimated by first determining the copula, either through parametric or nonparametric methods, and then employing the estimated copula density to evaluate the integrals and supremums that define the mixing coefficients.

\subsubsection{Definitions}
Let $\left(X_t\right)_{t\in\mathbb{N}}$ be a real stochastic process with a continuous marginal distribution function $F_t(x) = P(X_t \leq x)$ for all $t \in \mathbb{N}$. For any pair of time indices $s, t \in \mathbb{N}$, let $H_{st}(x, y) = P(X_s \leq x, X_t \leq y)$ represent the joint distribution function of the random vector $(X_s, X_t)$. According to Sklar's theorem (see Definition 1.3.1, Theorem 1.4.1 and Theorem 2.2.1 in \cite{durante2015principles}), the joint distribution $H_{st}$ can be expressed in terms of a copula $C_{st}$ and the marginal distribution functions $F_s$ and $F_t$
\[
H_{st}(x, y) = C_{st}(F_s(x), F_t(y)), \quad \forall \, x, y \in \mathbb{R}.
\]

The function $C_{st}: \mathbb{I} \times \mathbb{I} \to \mathbb{I}$ is called the copula of the random vector $(X_s, X_t)$, capturing the dependence structure between $X_s$ and $X_t$. Since $F_s$ and $F_t$ are continuous, the copula $C_{st}$ is uniquely defined.
\par
If we consider a stationary process i.e., when \(F_s = F_t = F\), hence using the copula \( C \), the conditional probability can be expressed as
\[
P(X_t \in A \mid X_{t-1} = x) = \frac{\partial C}{\partial u}(F(x), F(y)),
\]
where \( (u, v) = (F(x), F(y)) \) and $A = (-\infty, y]$. Notably, copulas are almost everywhere differentiable. Later we shall denote \(C_{,1}\) as the partial derivative \(\frac{\partial C}{\partial x}\). \par
For a continuous Markov process $(X_t)_{t\in\mathbb{N}}$ in discrete time, it has been established (see Darsow \emph{et al.}~\cite{Darsow1992}) that if $C_{t-1,t}$ is the copula of $(X_{t-1}, X_t)$ and $C_{t,t+1}$ is the copula of $(X_t, X_{t+1})$, then $C_{t-1,t} \ast C_{t,t+1}$ is the copula of $(X_{t-1}, X_{t+1})$ for $t\in\mathbb{N}$. The following theorem defines the Darsow product, ensuring that the product of two copulas is again a copula
\begin{theorem}[\cite{Darsow1992}]
  For all bivariate copulas $A$ and $B$, the function $A \ast B : \mathbb{I}^2 \to \mathbb{I}$ is defined as
\[
(A \ast B)(u, v) := \int_0^1 {\frac{\partial A}{\partial t} }(u, t) \, {\frac{\partial B}{\partial t} }(t, v) \, dt.
\]
Then $A \ast B$ is a bivariate copula. 
\end{theorem}

A stationary Markov chain \((X_t)_{t \in \mathbb{N}}\) is said to be generated by a marginal distribution $F$ and a copula $C$ if the joint distribution of \(X_0, X_1\) is
\[
H(x_1, x_2) = C(F(x_1), F(x_2)).
\]
Thus, a copula-based Markov chain is a stationary Markov chain represented by a copula of its consecutive states and its invariant distribution \(F\). If the chain is homogeneous, the copula of $(X_0, X_t)$ is given by:
\[
C^{\ast t} := C \ast C \ast \cdots \ast C \quad \text{(t times)},
\]
denoted by \(C^t\) for simplicity.

\subsubsection{Mixing Coefficients for Copula-based Markov Chains}\label{generatedchain}
The mixing coefficients for stationary Markov chains are invariant under strictly increasing and continuous transformations of the variables. For a continuous random variable \( X_t \) with distribution function \( F \), the transformation \( U_t = F(X_t) \) maps \( X_t \) to \( U_t \), which is uniformly distributed on \(\mathbb{I}\). Thus, the dependence structure of the process can be described entirely by copulas, simplifying the analysis.\par

For a stationary Markov chain \((U_t)_{t\in\mathbb{N}}\) with copula \(C\) and uniform marginals on \(\mathbb{I}\), the process \((F^{-1}(U_t))_{t\in\mathbb{N}}\) has marginal distribution \(F\) and the same copula \(C\). Note that for any \(u \in [0, 1]\), \(x \geq F^{-1}(u)\) if and only if \(F(x) \geq u\).

Since \(\sigma(F^{-1}(U_t)) \subseteq \sigma(U_t)\), the sigma-algebra generated by \(F^{-1}(U_t)\) is contained within that generated by \(U_t\). As a result, the dependence between \( F^{-1}(U_0) \) and \( F^{-1}(U_t) \) can only be weaker than or equal to that between \( U_0 \) and \( U_t \). Consequently, the mixing coefficients of \((F^{-1}(U_t))_{t\in\mathbb{N}}\) are bounded by those of \((U_t)_{t\in\mathbb{N}}\).\par

The transition probabilities of the Markov chain \((U_t)_{t\in\mathbb{N}}\) with uniform marginals and copula \( C \) are given by
\[
\mathbb{P}(U_t \leq y \mid U_0 = x) = C^t_{,1}(x, y) \quad \text{a.s.},
\]
and for any Borel set \( A \),
\[
P^t(x, A) = \mathbb{P}(U_t \in A \mid U_0 = x) = C^t_{,1}(x, A).
\]

The mixing coefficients for \((U_t)_{t \in \mathbb{N}}\), a stationary Markov chain with uniform marginals and copula \( C \), are
\[
\beta_t = \int_0^1 \sup_{B \in \mathcal{B}(\mathbb{I})} \left| C^t_{,1}(x, B) - \lambda(B) \right| dx,
\]
and
\[
\phi_t = \sup_{B \in \mathcal{B}(\mathbb{I})} \operatorname{ess} \sup_{x \in \mathbb{I}} \left| C^t_{,1}(x, B) - \lambda(B) \right|.
\]

If the copula \( C^t(x, y) \) is absolutely continuous with respect to \( \lambda^2 \), and its density is \( c_t(x, y) \), these coefficients become
\[
\beta_t = \int_0^1 \sup_{B \in \mathcal{B}(\mathbb{I})} \left| \int_B (c_t(x, y) - 1) \, dy \right| dx,
\]
and
\[
\phi_t = \sup_{B \subseteq R \cap \mathbb{I}} \operatorname{ess} \sup_{x \in \mathbb{I}} \left| \int_B (c_t(x, y) - 1) \, dy \right|.
\]
\subsubsection{Application in the setting of learning}
For a copula-based Markov chain \((X_t)_{t\in\mathbb{N}}\), as noted in Subsection \ref{generatedchain}, by transforming \( U_t = F(X_t) \), we have \( \sigma(F^{-1}(U_t)) \subseteq \sigma(U_t) \). Therefore, we consider a stationary Markov chain \((U_t)_{t\in\mathbb{N}}\) with copula \( C \) and uniform marginals on the space \( X = \mathbb{I} \), referring to it as a copula-generated Markov chain. We also define an output space \( Y = [-M, M] \) for some \( M > 0 \), where \( \rho \) is the joint stationary distribution on \( Z = \mathbb{I} \times [-M, M] \) with marginal \( F = \rho_X \) (e.g., the Lebesgue measure).

We know that due to \cite[Theorem 8]{MR2944418}, if \((X_t)_{t\in\mathbb{N}}\) is a stationary Markov chain generated by the copula \( C(x, y) \), such that the density of its absolutely continuous part satisfies \( c(x, y) \geq b > 0 \) a.s. on a set of Lebesgue measure 1, this condition ensures that \((X_t)_{t\in\mathbb{N}}\) is exponentially \(\phi\)-mixing and thereby geometrically ergodic. That is, there exist constants \( D > 0 \) and \( 0 < r < 1 \) such that 
\[
\beta_t \leq \phi_t \leq D r^t.
\] 

The above formulation applies well to a univariate time series \((X_t)_{t\in\mathbb{N}}\). However, in the context of learning, with a slight abuse of notation, we generally have a vector \((x_t, y_t)\) such that \(y_t = f(x_t)\) for some deterministic function \(f\). In the case of learning in an RKHS, this deterministic function translates to a regression function \(f_{\rho}\). The following lemma guarantees that the vector process \((z_t = (x_t, y_t))_{t\in\mathbb{N}}\) also preserves the mixing properties of \((x_t)_{t\in\mathbb{N}}\).

\begin{lemma}
If a sequence of random variables \((x_t)_{t\in\mathbb{N}}\) is a stationary Markov chain which is \(\phi\)-mixing at least exponentially fast defined on the probability space \((\Omega, \mathcal{F}, \mathbb{P})\) with state space \((X, \mathcal{B}(X))\), where \( X = \mathbb{I} \) is a topological space equipped with its Borel \(\sigma\)-algebra \(\mathcal{B}(X)\), then the process \((z_t = (x_t, f(x_t)))_{t\in\mathbb{N}}\), where \( f \) is a deterministic measurable function, is also a stationary Markov chain which is \(\phi\)-mixing at least exponentially fast.
\end{lemma}

\begin{proof}
Denote \( Z = X \times Y \) and \(\mathcal{B}(Z) = \mathcal{B}(X \times Y)\), the Borel \(\sigma\)-algebra on \( Z \). Here, \( Y \subseteq \mathbb{R} \) is defined as \( Y = f(X) \) for some deterministic measurable function \( f \).

Consider a measurable transformation \(\tau : X \rightarrow X \times Y\), such that for each \( x \in X \), \(\tau\) is given by
\[
\tau(x) = (x, f(x)).
\]
Then, we define a new sequence \((z_t)_{t\in\mathbb{N}}\) such that \( z_t = \tau(x_t) \).

Since \(\tau\) is one-to-one and measurable, the sequence \((z_t)_{t\in\mathbb{N}}\) also forms a Markov chain on \((X \times Y, \mathcal{B}(X \times Y))\) (also see Glover and Mitro~\cite{aa12a0dc-935e-364f-9676-02500a9f32ec}). Moreover, the \(\sigma\)-algebra generated by \( \tau(x_t) \), denoted by \(\sigma(\tau(x_t))\), satisfies
\[
\sigma(\tau(x_t)) = \sigma(z_t) \subseteq \sigma(x_t).
\]
This inclusion follows from the fact that \(\tau\) is a measurable transformation of \(x_t\).

Next, let \(\phi'_t\) and \(\phi_t\) denote the mixing coefficients of the chains \((z_t)_{t\in\mathbb{N}}\) and \((x_t)_{t\in\mathbb{N}}\), respectively. Because \(\sigma(\tau(x_t)) = \sigma(z_t) \subseteq \sigma(x_t)\), it follows that the mixing coefficient \(\phi'_t\) is bounded above by \(\phi_t\), i.e.,
\[
\phi'_t \leq \phi_t \quad \text{for all } t\in\mathbb{N}.
\]

Since \((x_t)_{t\in\mathbb{N}}\) is a \(\phi\)-mixing at least exponentially fast, we have
\[
\phi_t \leq D r^t, \quad \text{for some constants } D > 0 \text{ and } 0 < r < 1.
\]
Thus, it follows that
\[
\phi'_t \leq \phi_t \leq D r^t,
\]
implying that \((z_t)_{t\in\mathbb{N}}\) is also a \(\phi\)-mixing at least exponentially fast. In other words, the \(\phi'_t-\)mixing coefficient of the stationary Markov chain \((z_t)_{t\in\mathbb{N}}\) satisfies
\[
\phi'_t \leq D r^t, \quad \text{for some constants } D > 0 \text{ and } 0 < r < 1.
\]
\end{proof}
\section{Conclusion}
In this paper, we consider a strictly stationary Markov chain exhibiting exponential $\phi$- and $\beta$-mixing properties, characterized by the bounds $\beta_t \leq D_1r_1^t$ and $\phi_t \leq Dr^t$, where $D_1, D > 0$ and $0 < r_1, r < 1$. For such processes, we establish upper bounds for approximating a function using the observed realizations of the chain based on an online regularized learning algorithm. The exponential decay of the mixing coefficients, reflected in the parameters $r_1$ and $r$, plays a crucial role in determining the rate of decay of dependence. Specifically, for the case of finite samples, smaller values of $r_1$ and $r$ result in upper bounds for the sampling error that closely approach those of an i.i.d.~sample. We further extended the analysis on the rate of convergence for polynomially decaying \(\phi\)- and \(\beta\)-mixing sequences as well.
\par
Moreover, in all derived results, the upper bounds are expressed as the sum of the i.i.d.\ sample bound and an additional summand that captures the effect of the dependence structure in the process. Consequently, the upper bound for an i.i.d.\ sample arises as a special case of the general result, corresponding to the case where $D=0$ resp. $D_1=0$.

\section*{Acknowledgements}
This research was carried out under the Austrian COMET program (project S3AI with FFG no. 872172, \url{www.S3AI.at}, at SCCH, \url{www.scch.at}),  which is funded by the Austrian ministries BMK, BMDW, and the province of Upper Austria and the Federal Ministry for Climate Action, Environment, Energy, Mobility, Innovation and Technology (BMK),the Federal Ministry for Labour and Economy (BMAW), and the State of Upper Austria in the frame of the SCCH competence center INTEGRATE [(FFG grant no. 892418)] in the COMET - Competence Centers for Excellent Technologies Programme managed by Austrian Research Promotion Agency FFG.

\end{document}